\RequirePackage{fix-cm}

\let\origvec\vec

\let\vec\origvec

\documentclass[smallextended]{svjour3}       
\smartqed  
\usepackage{graphicx}

\usepackage{amsmath}

\usepackage{amsthm}

\usepackage{pmat}
\usepackage{subfig}
\usepackage{caption}

\usepackage{amssymb}

\usepackage[retainorgcmds]{IEEEtrantools}

\renewcommand{\theequationdis}{{\normalfont (\theequation)}}
\renewcommand{\theIEEEsubequationdis}{{\normalfont (\theIEEEsubequation)}}

\newtheorem{thm}{Theorem}
\newtheorem{lem}{Lemma}

\newtheorem*{ntn}{Notation}

\newtheorem*{rmk-2}{Remark}
\newtheorem{rmk-3}{Remark}
\newtheorem*{rmk-4}{Remark}
\newtheorem*{rmk-5}{Remark}
\newtheorem{rmk-6}{Remark}
\newtheorem{rmk-7}{Remark}
\newtheorem*{rmk-8}{Remark}
\newtheorem{cl}{Corollary}
\newtheorem{nt}{Note}

\begin{document}

\title{Interpretations of Deep Learning by Forests and Haar Wavelets}


\titlerunning{Interpretations of Deep Learning by Forests and Haar Wavelets}

\author{Changcun Huang 
}


\institute{Changcun Huang \at
}


\maketitle

\begin{abstract}
This paper presents a basic property of region dividing of ReLU (rectified linear unit) deep learning when new layers are successively added, by which two new perspectives of interpreting deep learning are given. The first is related to decision trees and forests; we construct a deep learning structure equivalent to a forest in classification abilities, which means that certain kinds of ReLU deep learning can be considered as forests. The second perspective is that Haar wavelet represented functions can be approximated by ReLU deep learning with arbitrary precision; and then a general conclusion of function approximation abilities of ReLU deep learning is given. Finally, generalize some of the conclusions of ReLU deep learning to the case of sigmoid-unit deep learning.
\end{abstract}

\section{Introduction}
\label{intro}
Deep leaning is nearly the most popular highlight of artificial intelligence nowadays and has made great successes in speech recognition \cite{[1]}, computer vision \cite{[2]}, playing game go \cite{[3]}, and so on. Despite its successful applications and history of nearly 40 years since Fukushima's paper in 1982 \cite{[4]}, the underlying principle still remains unclear, so that deep learning is often referred to as ``black box'', which greatly hinders its development.

One of the main concerns is why deep neural networks are more powerful than those with shallow layers. The answer to this question is the key of understanding deep learning, for which there are mainly three kinds of existing results: The first kind is specific, explaining particular class of functions realized by deep learning, such as \cite{[5],[6],[7],[8]}. The second kind \cite{[9],[10],[11],[12],[13]} is more general by studying the expressive ability of deep layers compared with shallow ones. The last one is about the function approximation ability of deep learning \cite{[14],[15]}.

Part of this paper belongs to the second kind. We'll present a basic property of region dividing of ReLU \cite{[16],[17]} deep learning when successively adding new layers (Section 3); and then realize the multi-category classification via a deep learning structure equivalent to a forest, which is a new interpretation of ReLU deep learning (Section 4).

The function approximation problem is also discussed (Section 5). We'll prove that Haar wavelet represented functions can be approximated by ReLU deep learning as precisely as possible. It follows a general result that ReLU deep learning can approximate an arbitrary continuous function on a closed set of $n$-dimensional space; the proof is totally different from \cite{[14],[15]}, giving a new perspective of interpreting ReLU deep learning.

Finally, one distinction of this paper is that some conclusions of ReLU deep learning can be generalized to the case of sigmoid-unit deep learning (Section 6).

Since ReLU has nearly become the dominant choice of neural units used by deep learning in recent years \cite{[9],[18]}, the main topics of this paper are general and useful both in theory and engineering.

\section{Preliminaries}
Before describing the region dividing property of deep learning, some preliminary results should be introduced first.
\subsection{Mechanisms of 3-layer networks} 
The discussion of 3-layer networks is the basis of comparisons between shallow and deep networks. And also, there exists 3-layer subnetworks in deep learning, in which the mechanism is the same as that of ordinary 3-layer networks.

We begin the discussion from a concrete example of two-category classification realized by a 3-layer network. It is well known that each ReLU corresponds to a hyperplane dividing the input space into two regions. In the case of two-dimensional input space, a hyperplane is reduced to a line. The following notes are applicable to the rest of this paper:
\begin{nt}
Hereafter, when referred to the classification by a hyperplane, the data points just being on hyperplanes are not taken into consideration.
\end{nt}
\begin{nt}
We'll not distinguish between the term of ``region dividing'' and that of ``data classification''; they usually do the same thing.
\end{nt}
\begin{nt}
For simplicity, all the figures of neural networks ignore the biases, which actually exist, however.
\end{nt}
\begin{nt}
Unless otherwise stated, the term of ``deep learning'' is the abbreviation of ``ReLU deep learning''.
\end{nt}

\begin{figure}[!t]
\captionsetup{justification=centering}
\centering
\subfloat[A 3-layer network.]{\includegraphics[width=1.7in]{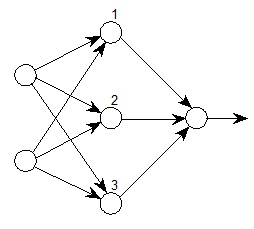}} \quad \quad
\subfloat[Region dividing with mutual interference.]{\includegraphics[width=1.75in]{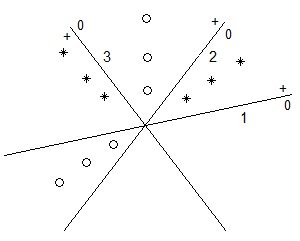}}
\caption{ The mechanism of 3-layer networks.}
\label{Fig.1}
\end{figure}

Fig.\ref{Fig.1} (a) is a 3-layer network with three ReLUs in the hidden layer denoted by 1, 2 and 3, corresponding to lines of 1, 2, and 3 of Fig.1 (b), respectively.

\begin{ntn}
We denote the two different sides of a hyperplane by ``$l$-$s$'', where $l$ is the index of the hyperplane and $s$ expresses the output of the ReLU with respect to this hyperplane. $l$-+ represents one side of hyperplane $l$, where the ReLU output is greater than 0; and the other side is denoted by $l$-0, where the ReLU output is zero.
\end{ntn}

For instance, in Fig.\ref{Fig.1}(b), 1-+ is the side above line 1 because the data in that half plane gives positive ReLU output, while 1-0 represents the below side producing zero outputs. The objective of the 3-layer network of Fig.\ref{Fig.1} (a) is to classify the data points of Fig.\ref{Fig.1} (b) into two categories: the output of the third layer should be 0 or 1 when the input sample belongs to ``o'' or ``$\ast$'' category, respectively. Output 1 can be obtained by normalizing the nonzero output of the ReLU.

In Fig.\ref{Fig.1} (b), we can add lines of 1, 2, and 3 one by one for classification. First, Line 1 is added, when the ``o'' samples below line 1 are correctly classified. The samples above line 1 should be further classified by more lines, such as line 2 and line 3. However, for example, when line 3 is added, the ``o'' samples below line 1 is simultaneously in the side of 3-+, producing nonzero outputs; that is to say, the subdividing of the half plane above line 1 by line 3 makes the ever correct classification result below line 1 change to be wrong, for which we may need to add other lines to eliminate the influence of line 3.

In fact, the final output expression of 3-layer networks (with a single output) is
\begin{equation}
y = f(\sum_{i = 1}^{N}w_is_i),
\end{equation}
where $s_i$ is the $i$th ReLU output of the hidden layer. If $s_i \ne 0$ and $w_i \ne 0$, the $i$th ReLU can influence the whole sum $\sum{w_is_i}$ by its nonzero output; in geometry language, it means that the $i$th hyperplane for region dividing will influence half of the input space where this ReLU output is nonzero. The influenced region may include ever correctly divided regions and the right results may be reversed. If the influence cannot be eliminated by adjusting present hyperplanes, new hyperplanes should be added. This procedure may occur recursively; hence the number of hyperplanes needed in 3-layer networks may be extremely larger than that it really needs, when we just want to divide the input space into separated regions without considering mutual influences. This is the general explanation of Fig.\ref{Fig.1}, from which a conclusion follows:
\begin{thm}
 In 3-layer networks, any new added ReLU of the hidden layer will influence half of the input space where the output of this ReLU is nonzero.
\end{thm}

We shall show that the interference of hyperplanes to each other can be avoided in deep learning.

\subsection{Transmitting of input-space regions through layers} 
In deep learning, the input space is only directly connected to the first hidden layer; how a region of the input space passes to subsequent layers is a key foundation of subregion dividing via a sequence of layers.

Pascanu et al. \cite{[9]} used ``intermediate layer'' to transmit an input-space region, which is actually by means of affine transforms; however, no general rigorous conclusions with proofs were presented. Although trivial in mathematics, due to great importance, we'll give detailed descriptions rigorously both in the conclusions and proofs about this problem, as well as add some necessary prerequisites for the establishing of the results.

\begin{lem}
Suppose that the input space $I$ is n-dimensional. The $n$ nonzero outputs of n ReLUs in the first hidden layer form a new space $H$. If the weight matrix $W$ of size $n \times n$ between the input layer and the first hidden layer is nonsingular, then $H$ is n-dimensional and is an affine transform of a region of $I$. The intersection of the nonzero-output areas of $n$ ReLUs in $I$ is the region to be transformed.
\end{lem}
\begin{proof}
We know that the nonzero output of a ReLU is $f(x) = x $ for $x > 0$. So an $n$-nonzero output vector $\vec{y}$ of $H$ can be written as
\begin{equation}
\vec{y} = W\vec{x} + \vec{b},
\end{equation}
where $\vec{x}$ is a vector of a certain region of $I$ and $\vec{b}$ is the bias vector of the $n$ ReLUs. (2) only combines the outputs of $n$ ReLUs to the matrix form. Obviously, (2) is an affine transform and if $W$ is nonsingular, the dimension of $H$ would be $n$.
\end{proof}

\begin{rmk-2}
The geometric meaning of Lemma 1: Nonsingular $W$ of {\rm (2)} implies non-parallel hyperplanes. Lemma 1 is equivalent to say that if the $n$ hyperplanes with respect to $n$ ReLUs are not parallel to each other, the space $H$ would be $n$-dimensional as well as an affine transform of a region of the input space.
\end{rmk-2}

\begin{thm}
In deep learning with $n$-dimensional input, if each succeeding layer has $n$ ReLUs with nonsingular input weight matrix, a certain region of the input space can be transmitted to subsequent layers one by one in the sense of affine transforms.
\end{thm}
\begin{proof}
The first hidden layer of Lemma 1 again can be considered as a new input space. By doing this recursively, a certain region of the initial input space can be transmitted to succeeding layers one by one in the sense of affine transforms, as long as this region is always in the nonzero parts of all the $n$ ReLUs in each layer.
\end{proof}

\section{Basic properties of region dividing via deep learning}

Section 2 has mentioned the mechanism of 3-layer networks that adding a new ReLU in hidden layer would influence half of the input space. However, this disadvantage can be avoided in deep learning; we'll describe a basic property of region dividing of deep learning in Theorem 3, which is the basis of the whole paper. The proof of Lemma 2 will be referred to for several times in the following sections to construct different types of deep learning structures.

\begin{figure}[!t]
\captionsetup{justification=centering}
\centering
\subfloat[A deep learning network.]{\includegraphics[width=2.0in]{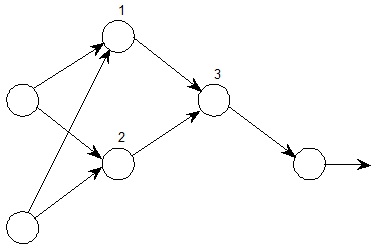}} \quad
\subfloat[Region dividing without mutual interference.]{\includegraphics[width=1.9in]{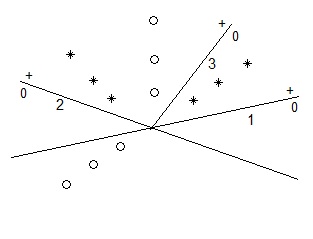}}
\caption{The mechanism of deep learning.}
\label{Fig.2}
\end{figure}

\subsection{The two-dimensional case}
Also begin with an example. Fig.\ref{Fig.2} is corresponding to Fig.\ref{Fig.1} of Section 2. In Fig.\ref{Fig.1}, to subdivide the region above 1-+, line 3 is added in the hidden layer; however, this operation influences the ever correctly classified results. In deep learning, the mutual interference among lines can be avoided by adding a new layer to restrict the influencing area of a line.

As shown in Fig.\ref{Fig.2} (b), first, line 1 is selected to divide the data points into two separate parts in different regions; then, we can always find line 2 having the same classification effect as line 1. In fact, when line 2 in Fig.\ref{Fig.2} (b) rotates counterclockwise towards line 1, the region between 1-+ and 2-+ (or between 1-0 and 2-0) can be as large as possible, such that all the data points above line 1 (or below line 1) are encompassed by 1-+ and 2-+ (or by 1-0 and 1-0); this is the way of finding line 2. Thus, all the data points are either in the region between 1-+ and 2-+ (denoted by region-+) or in the region between 1-0 and 2-0 (denoted by region-0).

Since line 1 and line 2 are not parallel to each other, by the remark of Lemma 1, the space of the two nonzero outputs of ReLU 1 and ReLU 2 of the first hidden layer is two-dimensional, as well as an affine transform of region-+; while region-0 is excluded from this layer in terms of zero outputs. Affine transforms do not affect the linear classification property of the data; so the linear classification in region-+ can be done in the space of the first hidden layer, without influencing region-0 because it has been excluded.

Now, instead of adding ReLU 3 in the same layer as ReLU 1 and ReLU 2 in Fig.\ref{Fig.1} (a), we add it in a new layer called the second hidden layer as shown in Fig.\ref{Fig.2} (a) to perform the classification of the first hidden layer. Correspondingly, in Fig.\ref{Fig.2} (b), line 3 should be added in a region of the first hidden layer, which is an affine transform of region-+ of the input space; however, this illustration is reasonable because the effect of linear classification is equivalent.


Obviously, the principle and operation underlying this example are general in two-dimensional space. In what follows, we shall directly generalize it to the $n$-dimensional case.


\subsection{The $n$-dimensional case}

\begin{lem}
For a 3-layer network with $n$-dimensional input, the hidden layer can be designed to realize an arbitrary linearly separable classification of two categories. One of the category will be excluded by the hidden layer, while the other one changes into its affine transform. Adding a new hidden layer can divide a selected region of the input space in the sense of affine transforms without influencing an excluded region.
\end{lem}
\begin{proof}


When the input space is $n$-dimensional, we need $n$ hyperplanes (ReLUs) to construct an $n$-dimensional space of the hidden layer, each of which realizes a same two-category classification. The function of those $n$ hyperplanes to be constructed is similar to that of line 1 and line 2 in Fig.\ref{Fig.2} (b).

First choose hyperplane 1 to divide the input space into two regions, containing the data points of category-0 and category-+, respectively; category-0 should be excluded, while category-+ may need to be subdivided. Then hyperplane 2 with the same classification effect as hyperplane 1 can be found by the similar method of the two-dimensional case. When hyperplane 2 rotates towards hyperplane 1 (counterclockwise or clockwise according to their relative positions), there exists infinite number of hyperplanes between them, all of which can classify the data in the same effect; choose $n-2$ of them as the remaining hyperplanes to construct an $n$-dimensional coordinate system. Since the $n$ selected hyperplanes are not parallel to each other, by the remark of Lemma 1, the $n$ nonzero outputs of $n$ ReLUs with respect to those hyperplanes form an $n$-dimensional linear space, which is an affine transform of a region of the input space; while the region giving $n$ zero outputs of the $n$ ReLUs will be excluded.

The constructed hidden layer has successfully excluded a region containing category-0 (region-0), as well as transmitted a region containing category-+ (region-+). If adding a new hidden layer, we can subdivide region-+ of the input space in the sense of affine transforms without influencing region-0.
\end{proof}

\begin{rmk-8}
The purpose of selecting $n$ non-parallel hyperplanes (ReLUs) is to construct an $n$-dimensional space to maintain the complete data structure of the $n$-dimensional input space in the sense of affine transforms. If the number of non-parallel hyperplanes is less than $n$, the outputs will be the subspace of the input space, which may lose information.
\end{rmk-8}

\begin{ntn}
Denote an arbitrary 2-layer subnetwork of a deep learning structure by $P$-$C$ with $n$-dimensional input, representing the previous layer and current layer, respectively; $W$ is the weight matrix between layer $P$ and layer $C$ as in (2).
\end{ntn}
\begin{thm}
In deep learning, if current layer $C$ has $n$ ReLUs with nonsingular input weight matrix $W$, adding ReLUs in a new layer $N$ after layer $C$ can divide a certain region of previous layer $P$ in the sense of affine transforms without influencing an excluded region. Similarly, adding new layers one by one can realize subregion dividing recursively; in each layer, data points that do not need to be subdivided can be put into the excluded region, so that the region dividing of succeeding layers will have no impact on them.
\end{thm}
\begin{proof}
The first part of the theorem is similar to Lemma 2. As long as $W$ is nonsingular, even if the $n$ ReLUs of layer $C$ are not specially designed, the region-transmitting property still holds. The rest of the proof is the recursive application of the first part.
\end{proof}

\begin{rmk-7}
Theorem 3 is a basic property of region dividing of certain kinds of ReLU deep learning, which is the key of this paper; all the following results are the consequences of this theorem.
\end{rmk-7}

\begin{rmk-7}
Theorem 3 indicates an advantage of deep layers for a type of deep learning structure. To classify complex data points, the deeper the network, the finer the subdividing will be. Once the step of adding new layers stops, the last three layers will perform the classification via the mechanism of 3-layer networks in a ultimate subregion.
\end{rmk-7}


\section{Multi-category classification of deep learning}
In this section, the multi-category classification ability of deep learning will be given in Theorem 4. Lemma 3 deals with the two-category case, while Theorem 4 is the repeated applications of Lemma 3.
\begin{lem}
For a finite number of data points composed of two categories in $n$-dimensional space, deep learning can classify them as a decision tree.
\end{lem}

\begin{figure}[!t]
\captionsetup{justification=centering}
\centering
\subfloat[A decision tree for two-category classification.]{\includegraphics[width=1.3in]{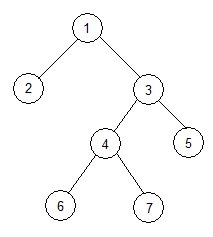}} \quad \quad
\subfloat[The deep learning structure corresponding to the left decision tree.]{\includegraphics[width=2.2in]{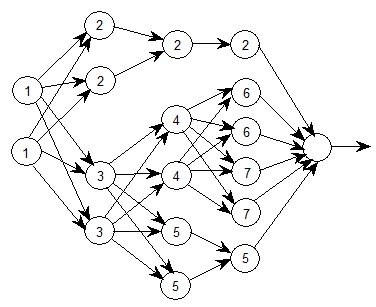}}
\caption{Two-category classification abilities of deep learning.}
\label{Fig.3}
\end{figure}

\begin{proof}
The proof is constructive by Lemma 2, Theorem 3 and the theory of decision trees. First, we can always construct a decision tree to realize this two-category classification, whose decision functions are linear classifiers. Second, there exists a deep learning structure equivalent to that decision tree, which is given by the following method.

As shown in Fig.\ref{Fig.3} (a), it's a four-level decision tree classifying two-dimensional data points and Fig.\ref{Fig.3} (b) is its corresponding deep learning structure. The layer of deep learning should correspond to the level of the decision tree except that the deep learning adds an output layer with one ReLU. Also in Fig.\ref{Fig.3}, the first layer, i.e., the input layer, has two ReLUs with respect to root node 1 of the decision tree because the input space is two-dimensional; the general case of $n$-dimensional space is similar.

In each layer, for the node having two child nodes, construct $2n$ ReLUs in the next layer: The $n$ of them (left child) separate the data points into region-+ and region-0, which are designed according to the decision function of this node by the method of Lemma 2; data points in region-+ can be subclassified by succeeding layers of child nodes without influencing region-0 excluded. The other $n$ ReLUs (right child) are different from the first group of $n$ ReLUs only in the parameter signs, respectively; they reverse the ReLU outputs of data points in region-+ and region-0, which instead makes region-0 to be subdivided. For example, in Fig.\ref{Fig.3}, node 1 has two child nodes, so that four ReLUs are needed in the next layer; two of them are for left child 2, while the other two are for right child 3. In the second layer, the weights and biases of ReLU 3's are opposite in the signs to those of ReLU 2's as well as with same absolute values, respectively.

For the leaf node, if the next layer is the last one, just connect its related ReLUs to the output ReLU, as node 6 and node 7 of Fig.\ref{Fig.3}. Otherwise, we should add one ReLU in each succeeding layer (except for the last one) to transmit the classification result to the last layer, such as node 2 and node 5 in Fig.\ref{Fig.3}; make sure that the weights and bias of the single ReLU of a leaf node in each layer maintain the nonzero output.

The weights and bias of the output-layer ReLU should be designed to distinguish between a left leaf node and a right leaf node. For instance, let the left leaf node and right leaf node of Fig.\ref{Fig.3} (a) correspond to zero output and nonzero output of deep learning of Fig.\ref{Fig.3} (b), respectively. The design is easy because in the layer previous the last one, when the output of a leaf node is nonzero, those of other leaf nodes will be mutually exclusive to be zero, due to the properties of decision trees. For example, in Fig.\ref{Fig.3} (b), when the output of ReLU 2 in the fourth layer (previous the last one) is nonzero, all the outputs of other ReLUs of this layer will be zero. So just consider that only ReLU 2 exists in that layer, by which the weight between ReLU 2 and the output-layer ReLU can be designed without influencing other leaf nodes. The bias of the output-layer ReLU can be set to zero because the weight itself is enough to produce the desired output. Since ReLU 2 is corresponding to a left leaf node, obviously, when the bias is zero, if the weight is set to a value less than or equal to zero, the design will meet the need. The general case is similar. This completes the constructing process.
\end{proof}


A forest is a sum of decision trees \cite{[19]}. We have proved that deep learning can realize the two-category classification as a decision tree. Next we'll show that deep learning can be equivalent to a forest in classification abilities.

\begin{figure}[!t]
\captionsetup{justification=centering}
\centering
\includegraphics[width=2.4in]{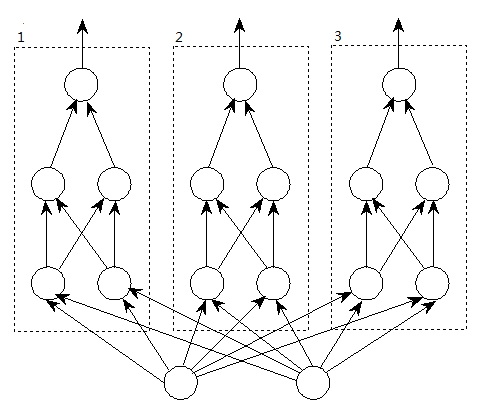}
\caption{An example of three-category classification.}
\label{Fig.4}
\end{figure}

\begin{thm}
Deep learning can classify arbitrary multi-category finite number of data points as a forest.
\end{thm}
\begin{proof}
The proof is to reduce the multi-category classification to the two-category case of Lemma 3 \cite{[20]}. Fig.\ref{Fig.4} is an example of three-category classification, in which each dotted rectangle classifies one of the three categories using the two-category method of Lemma 3. No matter how many categories should be classified, employ the two-category method to deal with each category separately and combine them into a whole deep learning structure.
\end{proof}

\begin{rmk-3}
From the viewpoint of forests, the layer depth of deep learning corresponds to the level of decision trees comprising a forest; while the number of neural units in each layer is related to that of the nodes in the corresponding level of the decision tree.
\end{rmk-3}

\begin{rmk-3}
Bengio et al. \cite{[19]} stated that decision trees are not easily generalized to variations of the training data, while forests do not have this limitation. By Theorem 4, deep learning can realize the function of forests and its generalization ability can be assured.
\end{rmk-3}


\section{Function approximations of deep learning}
There exists general results about the function approximation ability of 3-layer sigmoid-unit networks, such as Hecht-Nielsen \cite{[21]}, Cybenko \cite{[22]}, and Hornik et al.\cite{[23]}. Among them, Hecht-Nielsen's proof is constructive.

In the area of ReLU deep learning, Yarotsky \cite{[14]} and Liang et al. \cite{[15]} have also discussed such issues, both using similar methods. They first constructed deep learning structures to approximate polynomial functions; and then by Taylor series of smooth functions, the function approximation ability of deep learning was proved. Although the conclusions are assured, the deep learning structures they mentioned are constrained to certain types.

The proofs given blow are totally different from \cite{[14]} and \cite{[15]}, aiming at other types of deep learning structures, which will provide a new perspective of interpreting deep learning.

\begin{lem}
Any piecewise-constant function of Haar wavelets with finite number of building-block domains can be approximated by deep learning with arbitrary precision. If the input space is $n$-dimensional, $2n$ hidden layers are enough.
\end{lem}
\begin{proof}
The proof is based on Lemma 2 and Theorem 3. First prove the two-dimensional case. For a Haar wavelet represented function $f(x_1, x_2)$ defined on a closed set $S$ with finite building-block domains, we can always divide its domains into rectangles (or squares, similar hereafter) with different sizes and locations, each having a constant value (maybe the same with some other rectangles) of the function. The basic idea is to approximate the function by deep learning in each rectangle as precisely as possible. Because the number of rectangles is finite, if the approximation error for each rectangle is arbitrarily small, then the deep learning approximation to the whole function will be arbitrarily precise. So we just need to prove the case of one rectangle.

\begin{figure}[!t]
\captionsetup{justification=centering}
\centering
\subfloat[An isolated-rectangle domain separated by deep learning.]{\includegraphics[width=1.8in]{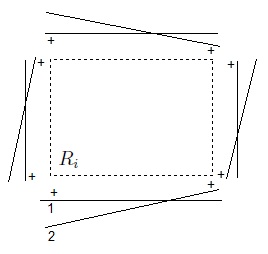}}
\subfloat[The case of adjacent rectangles.]{\includegraphics[width=1.8in]{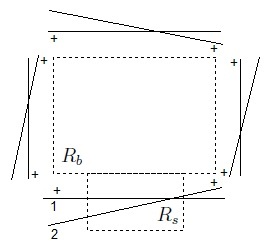}}
\subfloat[Corresponding deep learning structure.]{\includegraphics[width=1.0in]{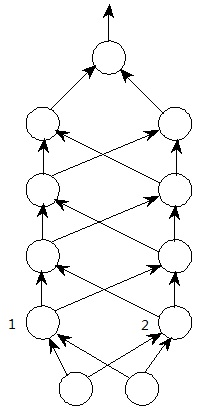}}
\caption{Function approximations of deep learning in one-rectangle domain.}
\label{Fig.5}
\end{figure}

First, for an isolated rectangle, such as $R_i$ in Fig.\ref{Fig.5} (a), it can be separated via deep learning. For each side of $R_i$, such as the bottom one, we can always find two lines (ReLUs) to divide some rectangle domains of $f(x_1, x_2)$ into two parts in two different regions, with one of the two lines parallel to the bottom side (such as line 1). $R_i$ is in the region where the outputs of the two ReLUs are both nonzero; all the rectangles below line 1 should be excluded by line 1 and line 2, and are in the other region (zero-output region). We see that the region between 1-+ and 2-0 or between 1-0 and 2-+ also gives nonzero output, which needs to be specially processed later different from the classification of discrete data points in Theorem 4.

After doing similar operations to the other sides, except for $R_i$, the rectangle domains of $f(x_1, x_2)$ are all excluded. However, the intersection of nonzero-output regions of the four separations is not $R_i$, but the region of the plane excluding four zero-output regions, which is a concave polygon (denoted by $P$) formed by eight lines such as in Fig.\ref{Fig.5} (a).

Note that only the separation of the bottom side of $R_i$ handled first is done in the input space of deep learning; the separations of three other sides should be done in the spaces of three succeeding layers, respectively, as shown in Fig.\ref{Fig.5} (c). However, the above operations are reasonable because of the properties of affine transforms. For example, if the second hidden layer of Fig.\ref{Fig.5} (c) corresponds to the separation of the left side of $R_i$, as long as we can find two lines for this side in the input space such as in Fig.\ref{Fig.5} (a),  the corresponding two lines in the space of the first hidden layer can also be found, with the parallel and collinear properties invariant. By the architecture of Fig.\ref{Fig.5} (c), the effects of four separations can be combined and finally only the data points in polygon $P$ can give nonzero outputs. The rest of the proof will not remind this related issue again.

In polygon $P$, let the output of deep learning be the value of the approximated function $f(x_1, x_2)$ in rectangle $R_i$. We now show that the limit of a sequence of $P$ can be $R_i$ by adjusting the parameters of eight lines. Denote an outer rectangle formed by four of the eight lines parallel to the four respective sides of $R_i$ (such as line 1) by $R_o$. For the separation of the bottom side of $R_i$, when line 2 rotates clockwise towards line 1 parallel to the bottom side, the limit of line 2 is line 1; during the rotating process, the classification effect of separating $R_i$ remains unchanged, while the region between 1-+ and 2-0 or between 1-0 and 2-+ becomes smaller and smaller. If we do the similar rotating operations to the cases of three other sides, the concave polygon $P$ can approximate $R_o$ by any desired precision.

When the outer rectangle $R_o$ shrinks to $R_i$, the polygon $P$ constructed by deep learning can also approximate $R_i$ with arbitrary precision; therefore, deep learning can approximate $f(x_1, x_2)$ in $R_i$ as precisely as possible.

Now discuss the case of adjacent rectangles. We call two rectangles adjacent when their two respective sides are on a same line. In Fig.\ref{Fig.5} (b), suppose that $f(x_1, x_2)$ has different constant values in the big rectangle $R_b$ and small rectangle $R_s$. As can be seen, the bottom side of $R_b$ shares a same line with the top side of $R_s$, so that they are adjacent. $R_b$ is to be separated and may have more than one adjacent rectangles; however, we only illustrate one of them, which is enough for the description of the proof.

As the case of an isolated rectangle, a concave polygon $P_b$ encompassing $R_b$ can be constructed by deep learning. In polygon $P_b$, the output of deep learning is normalized to the value of function  $f(x_1, x_2)$ in $R_b$. As shown in Fig.\ref{Fig.5} (b), part of $R_s$ is separated into $P_b$, where the output of deep learning is not equal to the actual function value in $R_s$. This type of approximation error occurs in all the adjacent rectangles separated into polygon $P_b$, where the function value is different from that of $R_b$. So the region of $P_b$ outside $R_b$ (denoted by $B$) is the source of approximation error of deep learning. Define the approximation error in $B$ as
\begin{equation}
E = \iint_{B'}(\hat{f}(x_1, x_2) - f(x_1, x_2)^2)^{1/2}dx_1dx_2,
\end{equation}
where $\hat{f}(x_1, x_2)$ is the approximating function of deep learning and $B'$ is a subset of region $B$ on which $f(x_1, x_2)$ is defined.

Let \begin{equation}
\omega = \max\limits_S\left|f(x_1', x_2') - f(x_1, x_2)\right|,
\end{equation}
where $S$ is the domain of $f(x_1, x_2)$ and $\omega$ is the maximum variation of $f(x_1, x_2)$, which always exists because $f(x_1, x_2)$ only has finite number of function values. Since the value of $\hat{f}(x_1, x_2)$ is also derived from $f(x_1, x_2)$, it's obvious that
\begin{equation}
E \le \omega S_B,
\end{equation}
where $S_B$ is the area of region $B$. Because the area of $P_b$ can be arbitrarily close to that of $R_b$, $S_B$ tends to be zero as $P_b \to R_b$; thus, $E$ can be as small as possible.

Fig.\ref{Fig.5} (c) is the structure of deep learning constructed for Fig.\ref{Fig.5} (a) or Fig.\ref{Fig.5} (b). The first hidden layer is corresponding to the region dividing by line 1 and line 2 with respect to the bottom side of a rectangle; and the succeeding three layers are the cases of three other sides. The four times of region dividing must be done in different layers successively to ensure that their effects can be combined. The final output should be normalized to the function value.

It is noted that four hidden layers for the 2-dimensional case are enough, since a rectangle has only four sides, each of which needs one hidden layer.

The whole structure of deep learning approximating $f(x_1, x_2)$ can be obtained by combining the subnetworks of all rectangle domains just like Fig.\ref{Fig.4}, each module of a dotted rectangle representing a certain rectangle domain of $f(x_1, x_2)$. This completes the proof of the two-dimensional case.

Similarly, the $n$-dimensional case can be proved. To approximate a single hyperrectangle, use $2n$ hidden layers (each for one of the $2n$ hyperrectangle sides) instead of four as in Fig.\ref{Fig.5} (c), with each layer having $n$ ReLUs. The rotating operations changing the parameters of hyperplanes can refer to the proof of Lemma 2. For each side of a hyperrectangle, $n$ hyperplanes are constructed to separate the hyperrectangle by the method of Lemma 2, with hyperplane 1 parallel to the side. Hyperplane 2 is second added and other $n-2$ hyperplanes are chosen between hyperplane 1 and hyperplane 2. So we just need to rotate hyperplane 2 as in the two-dimensional case, and then to insert other new $n-2$ hyperplanes between hyperplane 1 and the rotated hyperplane 2. The rest of the proof is trivial according to the two-dimensional case.
\end{proof}

\begin{thm}
Deep learning with $2n$ hidden layers can approximate any continuous function on a closed set of $n$-dimensional space with arbitrary precision.
\end{thm}
\begin{proof}
We know that Haar wavelets are capable of approximating continuous functions, while deep learning can approximate Haar wavelets as demonstrated in Lemma 4. This completes the proof.
\end{proof}

\begin{rmk-6}
From the perspective of Haar wavelets, the $2n$ hidden layers of deep learning are used to approximate hyperrectangle domains of Haar wavelet functions by the principle of Theorem 3; and the number of neural units in each layer corresponds to that of the hyperrectangle domains. In more detail, if there exists $m$ hyperrectangle domains, each hidden layer should use $m \times n$ neural units.
\end{rmk-6}

\begin{rmk-6}
The approximating accuracy of deep learning by Theorem 5 is determined by that of Haar wavelets. If we want the approximation error to be smaller, just make the Haar wavelet approximation to a function more precise, and then use deep learning to approximate the finer wavelets. The precise the Haar wavelets, the more ReLUs we need; however, the number of hidden layers is always $2n$.
\end{rmk-6}

\begin{rmk-6}
Lippmann \cite{[24]} ever gave a little similar proof about the classification ability of 3-layer networks composed of threshold logic units (TLUs). Although he didn't mention the function approximation problem, his region dividing by neural networks can accurately represent a Haar wavelet function. However, he only discussed the case of 3-layer networks with TLUs.
\end{rmk-6}

\section{Generalizations to sigmoid-unit deep learning}
Deep learning with sigmoid neural units has been successfully used in speech analysis \cite{[1]} and computer vision \cite{[2]}. It will be shown later that some related conclusions of ReLU deep learning of this paper can be generalized to the sigmoid-unit case.

\begin{cl}
All the conclusions of this paper about ReLU deep learning still hold in the case of a modified ReLU, which is
\begin{equation}
f(x) = \max(0, kx + b),
\end{equation}
where $k$ and $b$ are real with $k > 0$.
\end{cl}
\begin{proof}
(6) only changes the slope of the linear part and the position in $x$ axis of a ReLU; however, as long as a neural unit  has zero and linear outputs separated by a threshold, all the proofs related to the ReLU are applicable to the modified case of (6).
\end{proof}

\begin{cl}
In sigmoid-unit deep learning, a certain region of input space can be approximately transmitted to hidden layers by any desired precision in the sense of affine transforms.
\end{cl}
\begin{proof}
The derivative of sigmoid function $S(x)$ is $S'(x) = S(x)(1 - S(x))$, tending to 1/4 when $x \to 0$ ; that is to say, $S(x)$ is approximately a line of $y = x/4 + 1 / 2$ as precisely as possible when $x$ is close enough to zero. Thus, a certain segment of the sigmoid function can be approximately considered as a line. According to Corollary 1 and Theorem 2, this corollary holds.
\end{proof}
\begin{rmk-4}
In the classic paper \cite{[25]} of artificial neural networks, Hopfield also referred to the ``linear central region'' of $S(x)$ at $x = 0$ and used this approximately linear property to transmit information between nonlinear neurons. The thought is similar; however, the details are different from the background of applications.
\end{rmk-4}

\begin{cl}
Sigmoid-unit deep learning can exclude a certain region of the input space or a hidden layer space with arbitrary precision, so that region dividing in some other regions can not influence it.
\end{cl}
\begin{proof}
The sigmoid function $S(x)$ tends to zero as $x\to-\infty$, approximately corresponding to the zero-output part of a ReLU. Selecting probable parameters of sigmoid units can exclude a certain region as the case of ReLUs with arbitrary precision.
\end{proof}

\begin{rmk-5}
The above three corollaries suggest that sigmoid-unit deep learning can realize the function of ReLU deep leaning to some extent.
\end{rmk-5}

\section{Summary}
The region-dividing property of deep learning in Theorem 3 is general. On the basis of this property, we established the relationships between deep learning and forests, as well as between deep learning and Haar wavelets, by which the multi-category classification and function approximation abilities of deep learning were discussed.

All topics mentioned above are related to the ``black-box'' problem of deep learning, which is important both in theory and engineering. We hope that this paper will be helpful to this theme.



%
%

\end{document}